\documentclass[11pt]{article}

\usepackage{authblk}
\usepackage{setspace}
\usepackage{cprotect}
\usepackage{fullpage}
\usepackage[T1]{fontenc}    
\usepackage{hyperref}       
\usepackage{url}            
\usepackage{booktabs}       
\usepackage{amsfonts}       
\usepackage{nicefrac}       
\usepackage{microtype}      
\usepackage{graphicx}
\usepackage{subfigure}
\usepackage{booktabs} 
\usepackage{caption}
\usepackage{url}            
\usepackage{amsfonts}       
\usepackage{amsmath}
\usepackage{amssymb}
\usepackage{siunitx}
\usepackage{bbm}
\usepackage{multirow}
\usepackage{pifont}
\usepackage{hhline}
\usepackage{bm}
\usepackage{xcolor,colortbl}
\usepackage{mathtools}
\definecolor{Gray}{gray}{0.85}
\usepackage{makecell}
\usepackage{relsize}
\usepackage[colorinlistoftodos]{todonotes}

\newif\ifpaper
\papertrue

\def\b0{{0}}

\def\RR{\mathbb{R}}

\def\>{\rangle}
\def\rank{\operatorname{\mathop{rank}}}

\def\Set#1{\left\{ #1 \right\}}
\def\Bigbar#1{\mathrel{\left|\vphantom{#1}\right.}}
\def\Setbar#1#2{\Set{#1 \Bigbar{#1 #2} #2}}

\newtheorem{theorem}{Theorem}[section]
\newtheorem{lemma}[theorem]{Lemma}

\newtheorem{assumptions}[theorem]{Assumption}

\newenvironment{proof}{\par\noindent{\bf Proof:\ }}{\hfill$\Box$\\[2mm]}

\def\Span{\textrm{Span}}

\newcommand{\abs}[1]{\left|#1\right|}

\def\bydef{\mathrel{\mathop:}=}

\def\ones{\mathbf{1}}

\title{A Note on Connectivity of Sublevel Sets in Deep Learning}
\date{} 

\author{Quynh Nguyen\thanks{Email: quynhnguyenngoc89@gmail.com}
}

\begin{document}

\maketitle

\begin{abstract}
    It is shown that for deep neural networks, a single wide layer of width $N+1$ ($N$ being the number of training samples) 
    suffices to prove the connectivity of sublevel sets of the training loss function.
    In the two-layer setting, the same property may not hold even if one has just one neuron less 
    (i.e. width $N$ can lead to disconnected sublevel sets).
\end{abstract}

\section{Introduction}
Geometry of neural network loss landscape has been studied 
via the analysis of the global optimality of local minima \cite{Auer96,Choro15,Kawaguchi16,Quynh2017,Quynh2018b,Yun2019},
the existence of a continuous descending path to a global optimum \cite{Freeman2017,Quynh2019,Quynh2018c,SafSha2016,Venturi2019},
the connectivity of dropout-stable solutions \cite{kuditipudi2019explaining,shevchenko2020landscape},
and the topology of sublevel sets \cite{Quynh2019,Venturi2019}.
In this paper, we improve the result of \cite{Quynh2019}.

In particular, \cite{Quynh2019} shows that for a general class of convex loss functions (e.g. cross-entropy loss, square loss),
and for any training data with distinct samples, 
all the sublevel sets of the (empirical) training loss are connected if
the (deep) network satisfies the following conditions:
\begin{enumerate}
    \item The first hidden layer has $2N$ neurons ($N$ being the number of training samples), 
    and the other layers have non-increasing widths towards the output layer (a.k.a. {\em pyramidal network}).
    \item The activation is piecewise linear (excluding the linear ones), and strictly monotonic.
\end{enumerate}
On the one hand, we know that the connectivity of sublevel sets implies the following: 
(i) the loss surface has no bad local valleys, 
and that (ii) all the global minima are connected within a unique global valley.
On the other hand, we also know that for the property (i) to hold, a single layer of width $N$ suffices \cite{Quynh2018c,Venturi2019}.
This raises an intriguing question of how the loss function evolves when the number of neurons varies between $N$ and $2N$ neurons.
This paper aims to bridge this gap.
\paragraph{Main Contribution.}
We improve the width condition of the previous result from $2N$ to $N+1$ for deep architectures.
That is, if the width of the first layer is at least $N+1$, and the other assumptions hold as above, then every sublevel set of the empirical loss is connected.
Furthermore, in the two-layer setting, we show that the same statement is wrong
if one has one neuron less (i.e. with just $N$ neurons, the loss can have disconnected sublevel sets).
This shows that our width condition $N+1$ is tightest possible for this case unless additional assumptions on the data/network are made.

\paragraph{Preliminaries.} Consider an $L$-layer neural network where the widths of the layers are given by $(n_0,\ldots,n_L)$.
Here, $n_0$ and $n_L$ are input and output dimension respectively, and the rest are the hidden layers.
Let $N$ be the number of training samples.
Let $X\in\RR^{N\times n_0}$ and $Y\in\RR^{N\times n_L}$ be the training data.
The feature matrices $F_l\in\RR^{N\times n_l}$ are given by
\begin{align}\label{eq:def_feature_map}
    F_l=\begin{cases}
	    X & l=0,\\
	    \sigma(F_{l-1}W_l + \ones_N b_l^T) & l\in[L-1],\\
	    F_{L-1} W_L & l=L, 
        \end{cases}
\end{align}
where $W_l\in\RR^{n_{l-1}\times n_l}, b_l\in\RR^{n_l}$,
and $\sigma:\RR\to\RR$ is the activation function which is applied componentwise.
Let $\theta=(W_l,b_l)_{l=1}^{L}$ be the set of parameters.
The training loss is defined as
\begin{align*}
    \Phi(\theta) = \Psi(F_L(\theta), Y)
\end{align*}
where the function $\Psi(\hat{Y},Y)$ is assumed to be convex w.r.t. the first argument.

We consider the following class of activation functions.
\begin{assumptions}\label{ass:act1}
    $\sigma$ is strictly monotonic and $\sigma(\RR)=\RR.$
\end{assumptions}
\begin{assumptions}\label{ass:act2}
    There do not exist non-zero coefficients $(\lambda_i,a_i)_{i=1}^p$ with $a_i\neq a_j\,\forall\,i\neq j$ 
    such that $\sigma(x)=\sum_{i=1}^p\lambda_i\sigma(x-a_i)$ for every $x\in\RR.$
\end{assumptions}
Assumption \ref{ass:act2} is satisfied for every piecewise linear activation function of interest (e.g. ReLU and Leaky-ReLU), see \cite{Quynh2019}.
In the following, the $\alpha$-sublevel set of the training function is defined as
\begin{align}
     \forall\,\alpha\in\RR:\; L_\alpha=\Setbar{\theta}{\Phi(\theta)\leq\alpha}.
\end{align}
The $\alpha$-level set is defined as $\Setbar{\theta}{\Phi(\theta)=\alpha}.$
A subset $A\subseteq\RR^{n}$ is called connected if for any $x,y\in A$, there exists a continuous curve (path) $c:[0,1]\to A$ such that
it holds $c(0)=x$ and $c(1)=y.$

\section{Some Helpful Results}
\begin{theorem}\label{thm:lin_data}\cite{Quynh2019}
    Let Assumption \ref{ass:act1} hold, $\rank(X)=N$ and $n_1> \ldots> n_L$ where $L\geq 2.$
    Then,
    \begin{enumerate}
	\item Every sublevel set of $\Phi$ is connected.
	And $\Phi$ can attain any value arbitrarily close to $\inf_{\theta}\Phi(\theta).$
	\item Every non-empty connected component of every level set of $\Phi$ is unbounded.
    \end{enumerate}
\end{theorem}

\begin{lemma}\label{lem:full_rank_F}\cite{Quynh2019}
    Let $(X,W,b,V)\in\RR^{N\times d}\times\RR^{d\times n}\times\RR^n\times\RR^{n\times p}.$
    Let Assumption \ref{ass:act2} hold.
    Suppose that $n\geq N$ and $X$ has distinct rows.
    Let $Z=\sigma(XW+\ones_Nb^T)\,V.$
    There is a continuous curve $c:[0,1]\to\RR^{d\times n}\times\RR^n\times\RR^{n\times p}$ 
    with $c(\lambda)=(W(\lambda),b(\lambda),V(\lambda))$ satisfying:
    \begin{enumerate}
	\item $c(0)=(W,b,V).$
	\item $\sigma\Big(XW(\lambda))+\ones_Nb(\lambda)^T\Big)\,V(\lambda)=Z,\,\forall\,\lambda\in[0,1].$
	\item $\rank\Big(\sigma\Big(XW(1)+\ones_Nb(1)^T\Big)\Big)=N.$
    \end{enumerate}
\end{lemma}

\begin{lemma}\label{lem:FW}
    Let $F\in\RR^{N\times n}, W\in\RR^{n\times p}.$
    Let $\mathcal{I}$ be a subset of columns of $F$, and $\bar{\mathcal{I}}$ its complement such that
    every column in $\bar{\mathcal{I}}$ belongs to the linear span of all the columns in $\mathcal{I}.$
    Then there exists a continuous curve $c:[0,1]\to\RR^{n\times p}$
    which satisfies the following:
    \begin{enumerate}
	\item $c(0)=W$ and $Fc(\lambda)=FW,\,\forall\,\lambda\in[0,1].$
	\item Let $W'=c(1).$ Then $W'(\bar{\mathcal{I}},:)=0.$
    \end{enumerate}
\end{lemma}
\begin{proof}
    Let $r=\rank(F)<n.$ 
    There exists $E\in\RR^{r\times(n-r)}$ 
    so that $F(:,\bar{\mathcal{I}})=F(:,\mathcal{I})\,E.$
    Let $P\in\RR^{n\times n}$ be a permutation matrix which permutes the columns of $F$ according to $\mathcal{I}$ so that 
    we can write $F=[F(:,\mathcal{I}),F(:,\bar{\mathcal{I}})]\,P.$
    Consider the continuous curve $c:[0,1]\to\RR^{n\times p}$ defined as
    \begin{align*}
	c(\lambda)=P^T\,\begin{bmatrix}W(\mathcal{I},:)+\lambda E\,W(\bar{\mathcal{I}},:)\\ (1-\lambda)W(\bar{\mathcal{I}},:)\end{bmatrix}, \,\forall\,\lambda\in[0,1].
    \end{align*}
    It holds $c(0)=P^T\,\begin{bmatrix}W(\mathcal{I},:)\\W(\bar{\mathcal{I}},:)\end{bmatrix}=W.$
    For every $\lambda\in[0,1]:$
    \begin{align*}
	Fc(\lambda) 
	&= [F(:,\mathcal{I}),F(:,\bar{\mathcal{I}})]\,PP^T\, \begin{bmatrix}W(\mathcal{I},:)+\lambda E\,W(\bar{\mathcal{I}},:)\\ (1-\lambda)W(\bar{\mathcal{I}},:)\end{bmatrix}\\
	&= F(:,\mathcal{I}) W(\mathcal{I},:) + F(:,\bar{\mathcal{I}}) W(\bar{\mathcal{I}},:) = FW .
    \end{align*}
    The second statement follows by noting that $c(\lambda)(\bar{\mathcal{I}},:)=(1-\lambda)W(\bar{\mathcal{I}},:).$
\end{proof}

\begin{lemma}\label{lem:permute}
    Let $F\in\RR^{N\times n}, W\in\RR^{n\times p}.$
    Let $k,j\in[n]$ be a pair of distinct columns of $W$ such that $W_{j:}=0$ and $F_{:j}=F_{:k}.$
    Then there exists a continuous curve $c:[0,1]\to\RR^{n\times p}$ 
    which satisfies that:
    \begin{enumerate}
	\item $c(0)=W$ and $Fc(\lambda)=FW,\,\forall\,\lambda\in[0,1].$
	\item Let $U=c(1).$ Then $c(1)(k,:)=U_{k:}=0.$
    \end{enumerate}
\end{lemma}
\begin{proof}
    Take the path $c$ to be
    \begin{align}
	c(\lambda)(p,:)
	\begin{cases}
	    W_{p:} & p\neq j,k\\
	    (1-\lambda) W_{k:} & p=k\\
	    \lambda W_{k:} & p=j\\
	\end{cases}
    \end{align}
    Then, it is easy to check that $c$ satisfies the lemma.
\end{proof}

\section{Sufficiency of Width $N+1$ for Deep Networks}
Our main result is the following.
\begin{theorem}\label{thm:connected_sublevel_sets}
    Consider an $L$-layer network.
    Let Assumption \ref{ass:act1} and Assumption \ref{ass:act2} hold.
    Suppose that $\boxed{n_1\geq N+1}$ and $n_2>\ldots> n_L.$
    Then every sublevel set of $\Phi$ is connected,
    and every non-empty connected component of every level set of $\Phi$ is unbounded.
\end{theorem}
\begin{proof}
    Let $L_\alpha=\Setbar{\theta}{\Phi(\theta)\leq\alpha}.$
    Let $\theta=(W_l,b_l)_{l=1}^L,\theta'=(W_l',b_l')_{l=1}^L$ be arbitrary points in $L_\alpha.$
    We want to show that there is a connected path between $\theta$ and $\theta'$
    on which the loss is not larger than $\alpha.$
    The output at the first layer is given by
    \begin{align*}
	F_1\bydef F_1(\theta)&=\sigma([X,\ones_N][W_1^T,b_1]^T),\\
	F_1'\bydef F_1(\theta')&=\sigma([X,\ones_N][W_1'^T,b_1']^T).
    \end{align*}
    First, by applying Lemma \ref{lem:full_rank_F} to $(X,W_1,b_1,W_2)$,
    we can assume that $F_1$ has full rank.
    Otherwise, there is a path of constant loss from $\theta$ to some other point where this property is satisfied.
    Similarly, one can also assume that $\rank(F_1')=N.$
    In the remaining, we assume w.l.o.g. that the first $N$ columns of $F_1'$ are linearly independent.
    Let $\tilde{W}_1=[W_1^T,b_1]^T.$
    
    As the second step, we fix $\theta'$ and move $\theta$ along another path of constant loss
    such that at the end of this path, $\theta$ and $\theta'$ have the same parameter values at the first layer.
    To do so, we use induction. 
    Assume that we have already made the first $k-1$ columns of $\tilde{W}_1$ coincide with $\tilde{W}_1'$.
    Let us show how to do this for the $k$-th column.
    Let $j'$ be the smallest index such that $j'\geq k$ and $(F_1)_{:j'}\in\Span\Set{(F_1)_{:p}}_{p=1}^{j'-1}.$
    Such an index always exists because of the following reason: 
    for $k-1\geq N$, this simply follows from the fact that the first $N$ columns of $F_1$ are linearly independent by induction assumption;
    and for $k-1<N$, it follows from the fact that 
    there are at most $N$ independent columns but we have more columns to choose than $N$ (i.e. $n_1>N$).
    Now, by the choice of $j'$, we have that $(F_1)_{:j'}$ belongs to the linear span of the first $j'-1$ columns.
    Thus by using Lemma \ref{lem:FW}, we can find a path on $W_2$ such that the output stays invariant and we obtain at the end of the path $(W_2)_{j':}=0.$
    Now, since the outcoming weights of neuron $j'$ is zero, we can continuously change its incoming weights 
    to any target value we want without affecting the output/loss. 
    Consider to do this so that we obtain $(\tilde{W}_1)_{:j'}=(\tilde{W}_1)_{:k}.$ 
    This gives us $(F_1)_{:k}=(F_1)_{:j'}.$
    Now, the neurons $\Set{k,j'}$ satisfy the conditions of Lemma \ref{lem:permute}, thus
    we can follow a path of constant loss to obtain $(W_2)_{k:}=0.$
    Finally, take the direct line segment between $(\tilde{W}_1)_{:k}$ and $(\tilde{W}_1')_{:k}$ to obtain $(\tilde{W}_1)_{:k}=(\tilde{W}_1')_{:k}$.
    
    The completion of the second step leaves us with $(W_1,b_1)=(W_1',b_1').$
    Note that $\theta'$ is unchanged in the second step, so we still have from the earlier construction that $\rank(F_1(\theta'))=N.$
    Let $F_1\bydef F_1(\theta)=F_1(\theta').$
    Then, by fixing $(W_1,b_1)$, one can view $F_1$ as the new training data for the subnetwork from layer $1$ till layer $L$.
    This subnetwork and the new data  $F_1$ satisfy all the conditions of Theorem \ref{thm:lin_data}, 
    and so it follows that the loss function restricted to this subnetwork has connected sublevel sets.
    That means that there is a connected path between $(W_l,b_l)_{l=2}^L$ and $(W_l',b_l')_{l=2}^L$
    on which the loss is not larger than $\alpha.$
    This further implies that a continuous path between $\theta$ and $\theta'$ exists in $L_\alpha$, 
    and so $L_\alpha$ must be connected.
\end{proof}

\section{Necessity of Width $N+1$ for Two-layer Networks}
The next result shows that a two-layer model with width $N$ may not have connected sublevel sets.
\begin{theorem}
    Consider a two-layer network where the width of the hidden layer satisifes $n_1\leq N.$
    Let Assumption \ref{ass:act2} hold.
    Let $(X,Y)$ be some training data such that the following conditions are satisfied: 
    (i) there exists $n_1$ samples indexed by the subset $\mathcal{I}\subseteq[n_1]$ with $\abs{\mathcal{I}}=n_1$, 
    such that $\Set{Y_{i:}}_{i\in\mathcal{I}}$ are linearly independent,
    (ii) the network can fit perfectly the training data, i.e. $\exists\theta=(W_1,b_1,W_2) \textrm{ s.t. } \sigma(XW_1+\ones_Nb_1^T)W_2=Y.$
    Then, the set of global minima of $\Phi$ is disconnected.
\end{theorem}
\begin{proof}
    Let $\theta$ be a global minimum.
    Recall that $F_1(\theta)=\sigma(XW_1+\ones_Nb_1^T).$
    Let us omit the argument and write just $F_1.$
    As $\theta$ is optimal, we have $(F_1)_{\mathcal{I},:}W_2=Y_{\mathcal{I},:}$.
    By assumption, all the $n_1$ rows of $Y_{\mathcal{I},:}$ are linearly independent, so
    the matrix $(F_1)_{\mathcal{I},:}$ must have full rank $n_1.$
    As a remark, the existence of points in parameter space where $(F_1)_{\mathcal{I},:}$ has full rank can be guaranteed by Lemma \ref{lem:full_rank_F}.
    Let $\theta'$ be obtained from $\theta$ by exchanging two neurons.
    Then, clearly $\theta'$ is also a global minimizer.
    Let $F_1'=F_1(\theta').$
    Then, $(F_1')_{\mathcal{I},:}$ can be obtained from $(F_1)_{\mathcal{I},:}$ by exchanging two columns, so their determinants have opposite signs.
    Now, suppose by contradiction that $\theta$ and $\theta'$ are connected by a path of global minima $\theta(\lambda), \lambda\in[0,1].$
    By the above argument, it holds $\rank((F_1(\theta(\lambda)))_{\mathcal{I},:})=n_1.$
    This implies that there is a continuous path from $(F_1)_{\mathcal{I},:}$ to $(F_1')_{\mathcal{I},:}$ along which 
    the matrix has full rank.
    This contradicts the fact that the set of full-rank matrices in $\RR^{n_1\times n_1}$ has two connected components, one with positive determinant and the other with negative determinant \cite{Evard1994}.
\end{proof}

\section{Acknowledgement}
I want to thank Peter Bartlett for raising the intriguing question regarding the transition of the loss surface for intermediate widths 
which leads to this manuscript.

\bibliography{regul}

\begin{thebibliography}{10}

\bibitem{Auer96}
Peter Auer, Mark Herbster, and Manfred K.~K. Warmuth.
\newblock Exponentially many local minima for single neurons.
\newblock In {\em NIPS}, 1996.

\bibitem{Choro15}
Anna Choromanska, Mikael Henaff, Michael Mathieu, Gerard~Ben Arous, and Yann
  LeCun.
\newblock The loss surfaces of multilayer networks.
\newblock In {\em AISTATS}, 2015.

\bibitem{Evard1994}
Jean-Claude Evard and Farhad Jafarii.
\newblock The set of all mxn rectangular real matrices of rank-r is connected
  by analytic regular arcs.
\newblock In {\em Proceedings of American Mathematical Society}, 1994.

\bibitem{Freeman2017}
Daniel~C. Freeman and Joan Bruna.
\newblock Topology and geometry of half-rectified network optimization.
\newblock In {\em ICLR}, 2017.

\bibitem{Kawaguchi16}
Kenji Kawaguchi.
\newblock Deep learning without poor local minima.
\newblock In {\em NIPS}, 2016.

\bibitem{kuditipudi2019explaining}
Rohith Kuditipudi, Xiang Wang, Holden Lee, Yi~Zhang, Zhiyuan Li, Wei Hu, Rong
  Ge, and Sanjeev Arora.
\newblock Explaining landscape connectivity of low-cost solutions for
  multilayer nets.
\newblock In {\em NIPS}, 2019.

\bibitem{Quynh2019}
Quynh Nguyen.
\newblock On connected sublevel sets in deep learning.
\newblock In {\em ICML}, 2019.

\bibitem{Quynh2017}
Quynh Nguyen and Matthias Hein.
\newblock The loss surface of deep and wide neural networks.
\newblock In {\em ICML}, 2017.

\bibitem{Quynh2018b}
Quynh Nguyen and Matthias Hein.
\newblock Optimization landscape and expressivity of deep cnns.
\newblock In {\em ICML}, 2018.

\bibitem{Quynh2018c}
Quynh Nguyen, Mahesh~C. Mukkamala, and Matthias Hein.
\newblock On the loss landscape of a class of deep neural networks with no bad
  local valleys.
\newblock In {\em ICLR}, 2019.

\bibitem{SafSha2016}
Itay Safran and Ohad Shamir.
\newblock On the quality of the initial basin in overspecified networks.
\newblock In {\em ICML}, 2016.

\bibitem{shevchenko2020landscape}
Alexander Shevchenko and Marco Mondelli.
\newblock Landscape connectivity and dropout stability of sgd solutions for
  over-parameterized neural networks.
\newblock In {\em ICML}, 2020.

\bibitem{Venturi2019}
Luca Venturi, Afonso~S. Bandeira, and Joan Bruna.
\newblock Spurious valleys in two-layer neural network optimization landscapes.
\newblock {\em JMLR}, 2019.

\bibitem{Yun2019}
Chulhee Yun, Suvrit Sra, and Ali Jadbabaie.
\newblock Small nonlinearities in activation functions create bad local minima
  in neural networks.
\newblock In {\em ICLR}, 2019.

\end{thebibliography}
\bibliographystyle{plain}

\end{document}